\documentclass{article} 
\usepackage{times,epsfig,wrapfig,amsfonts,bm,color,enumitem,algorithm,algpseudocode}
\usepackage{amsmath,amssymb,amsthm}

\newcommand{\argmin}[1]{\underset{#1}{\mathrm{argmin}} \:}

\newcommand{\inner}[2]{\left\langle #1, #2 \right\rangle}

\newtheorem{theorem}{Theorem}

\newtheorem{corollary}{Corollary}

\newtheorem*{theorem*}{Theorem}

\newcommand{\norm}[1]{\left\lVert{#1}\right\rVert}
\newcommand{\abs}[1]{\left\lvert{#1}\right\rvert}

\newcommand{\R}{\mathbb{R}}

\usepackage{iclr2015,times}
\usepackage{hyperref}
\usepackage{enumitem}
\usepackage{url}
\newcommand{\T}{^\top}
\newcommand{\RR}{\mathbb{R}}

\newcommand{\vu}{\boldsymbol{u}}
\newcommand{\vv}{\boldsymbol{v}}
\newcommand{\vx}{\boldsymbol{x}}
\newcommand{\vy}{\boldsymbol{y}}
\newcommand{\mU}{\boldsymbol{U}}
\newcommand{\mV}{\boldsymbol{V}}
\newcommand{\mW}{\boldsymbol{W}}
\newcommand{\calU}{\mathcal{U}}

\newcommand{\calD}{\mathcal{D}}

\newcommand{\calA}{\mathcal{A}}
\title{In Search of the Real Inductive Bias:\\
On the Role of Implicit Regularization in Deep Learning}

\author{
Behnam Neyshabur, Ryota Tomioka \& Nathan Srebro\\
Toyota Technological Institute at Chicago\\
Chicago, IL 60637, USA \\
\texttt{\{bneyshabur,tomioka,nati\}@ttic.edu} \\
}

\iclrfinalcopy 

\begin{document}
\maketitle

\begin{abstract}
  We present experiments demonstrating that some other form of
  capacity control, different from network size, plays a central role
  in learning multi-layer feed-forward networks.  We argue,
  partially through analogy to matrix factorization, that this is an
  inductive bias that can help shed light on deep learning.
\end{abstract}

\section{Introduction}
Central to any form of learning is an inductive bias that induces some
sort of capacity control (i.e.~restricts or encourages predictors to
be ``simple'' in some way), which in turn allows for generalization.
The success of learning then depends on how well the inductive bias
captures reality (i.e.~how expressive is the hypothesis class of
``simple'' predictors) relative to the capacity induced, as well as on
the computational complexity of fitting a ``simple'' predictor to the
training data.

Let us consider learning with feed-forward networks from this
perspective.  If we search for the weights minimizing the training
error, we are essentially considering the hypothesis class of
predictors representable with different weight vectors, typically for
some fixed architecture.  Capacity is then controlled by the size
(number of weights) of the network\footnote{The exact correspondence
  depends on the activation function---for hard thresholding
  activation the pseudo-dimension, and hence sample complexity, scales
  as $O(S \log S)$, where $S$ is the number of weights in the network. 
  With sigmoidal activation it is between $\Omega(S^2)$ and
  $O(S^4)$ \citep{anthony1999}.}.  Our justification for using such networks is
then that many interesting and realistic functions can be represented
by not-too-large (and hence bounded capacity) feed-forward networks.
Indeed, in many cases we can show how specific architectures can
capture desired behaviors.  More broadly, any $O(T)$ time computable
function can be captured by an $O(T^2)$ sized network, and so the
expressive power of such networks is indeed great~\cite[Theorem~9.25]{sipser2012}.

At the same time, we also know that learning even moderately sized
networks is computationally intractable---not only is it NP-hard to
minimize the empirical error, even with only three hidden units, but
it is hard to learn small feed-forward networks using {\em any}
learning method (subject to cryptographic assumptions).  That is, even
for binary classification using a network with a single hidden layer and a logarithmic (in the
input size) number of hidden units, and even if we know the true
targets are {\em exactly} captured by such a small network, there is
likely no efficient algorithm that can ensure error better than 1/2
\citep{klivans2006cryptographic,Daniely14}---not if the algorithm
tries to fit such a network, not even if it tries to fit a much larger
network, and in fact no matter how the algorithm represents
predictors (see the Appendix).  And so, merely knowing that some not-too-large
architecture is excellent in expressing reality does {\em not} explain
why we are able to learn using it, nor using an even larger network.
Why is it then that we succeed in learning using multilayer
feed-forward networks?  Can we identify a property that makes them
possible to learn?  An alternative inductive bias?

Here, we make our first steps at shedding light on this question by
going back to our understanding of network size as the capacity
control at play.  

Our main observation, based on empirical experimentation with
single-hidden-layer networks of increasing size (increasing number of
hidden units), is that size does {\em not} behave as a capacity
control parameter, and in fact there must be some other, implicit,
capacity control at play.  We suggest that this hidden capacity
control might be the real inductive bias when learning with deep
networks.

In order to try to gain an understanding at the possible inductive
bias, we draw an analogy to matrix factorization and understand
dimensionality versus norm control there.  Based on this analogy we
suggest that implicit norm regularization might be central also for
deep learning, and also there we should think of infinite-sized
bounded-norm models.  We then also demonstrate how (implicit) $\ell_2$
weight decay in an infinite two-layer network gives rise to a ``convex
neural net'', with an infinite hidden layer and $\ell_1$ (not $\ell_2$)
regularization in the top layer.

\section{Network Size and Generalization}\label{sec:netsize}

Consider training a feed-forward network by finding the weights
minimizing the training error.  Specifically, we will consider a
network with $d$ real-valued inputs $\vx=(x[1],\ldots,x[d])$, a single
hidden layer with $H$ rectified linear units, and $k$ outputs
$y[1],\ldots,y[k]$,
\begin{equation}
  \label{eq:yk}
  y[j] = \sum_{h=1}^H v_{hj} [\inner{\vu_h}{\vx}]_+
\end{equation}
where $[z]_{+}:=\max(z,0)$ is the rectified linear activation function
and $\vu_h\in\R^d, v_{hj}\in\R$ are the weights learned by minimizing
a (truncated) soft-max cross entropy loss\footnote{When using soft-max
  cross-entropy, the loss is never exactly zero for correct
  predictions with finite margins/confidences.  Instead, if the data
  is seperable, in order to minimize the loss the weights need to be
  scaled up toward infinity and the cross entropy loss goes to zero,
  and a global minimum is never attained.  In order to be able to say
  that we are actually reaching a zero loss solution, and hence a
  global minimum, we use a slightly modified soft-max which does not
  noticeably change the results in practice.  This truncated loss
  returns the same exact value for wrong predictions or correct
  prediction with confidences less than a threshold but returns zero
  for correct predictions with large enough margins: Let
  $\{s_i\}_{i=1}^k$ be the scores for $k$ possible labels and $c$ be
  the correct labels. Then the soft-max cross-entropy loss can be
  written as $\ell(s,c) = \ln \sum_{i} \exp(s_i - s_c)$ but we instead
  use the differentiable loss function $\hat{\ell}(s,c) = \ln \sum_{i}
  f(s_i-s_c)$ where $f(x)=\exp(x)$ for $x\geq -11$ and $f(x)
  =\exp(-11) [x+13]_+^2/4$ otherwise. Therefore, we only deviate from
  the soft-max cross-entropy when the margin is more than $11$, at
  which point the effect of this deviation is negligible (we always
  have $\abs{\ell(s,c)-\hat{\ell}(s,c)}\leq 0.000003k$)---if there are
  any actual errors the behavior on them would completely dominate
  correct examples with margin over $11$, and if there are no errors
  we are just capping the amount by which we need to scale up the
  weights.} on $n$ labeled training examples.  The total number of
weights is then $H(d+k)$.

What happens to the training and test errors when we increase the
network size $H$? The training error will necessarily decrease.  The
test error might initially decrease as the approximation error is
reduced and the network is better able to capture the targets.
However, as the size increases further, we loose our capacity control
and generalization ability, and should start overfitting.  This is the
classic approximation-estimation tradeoff behavior.

Consider, however, the results shown in Figure \ref{fig:intro}, where
we trained networks of increasing size on the MNIST and CIFAR-10
datasets.  Training was done using stochastic gradient descent with
momentum and diminishing step sizes, on the training error and without
any explicit regularization.  As expected, both training and test
error initially decrease.  More surprising is that if we increase the
size of the network past the size required to achieve zero training
error, the test error continues decreasing!  This behavior is not at
all predicted by, and even contrary to, viewing learning as fitting a
hypothesis class controlled by network size.  For example for MNIST, 32 units
are enough to attain zero training error.  When we allow more units,
the network is not fitting the training data any better, but the
estimation error, and hence the generalization error, should increase
with the increase in capacity.  However, the test error goes down.  In
fact, as we add more and more parameters, even beyond the number
of training examples, the generalization error does not go up.

\begin{figure*}[t!]
\hbox{ \centering
\setlength{\epsfxsize}{0.455\textwidth}
\epsfbox{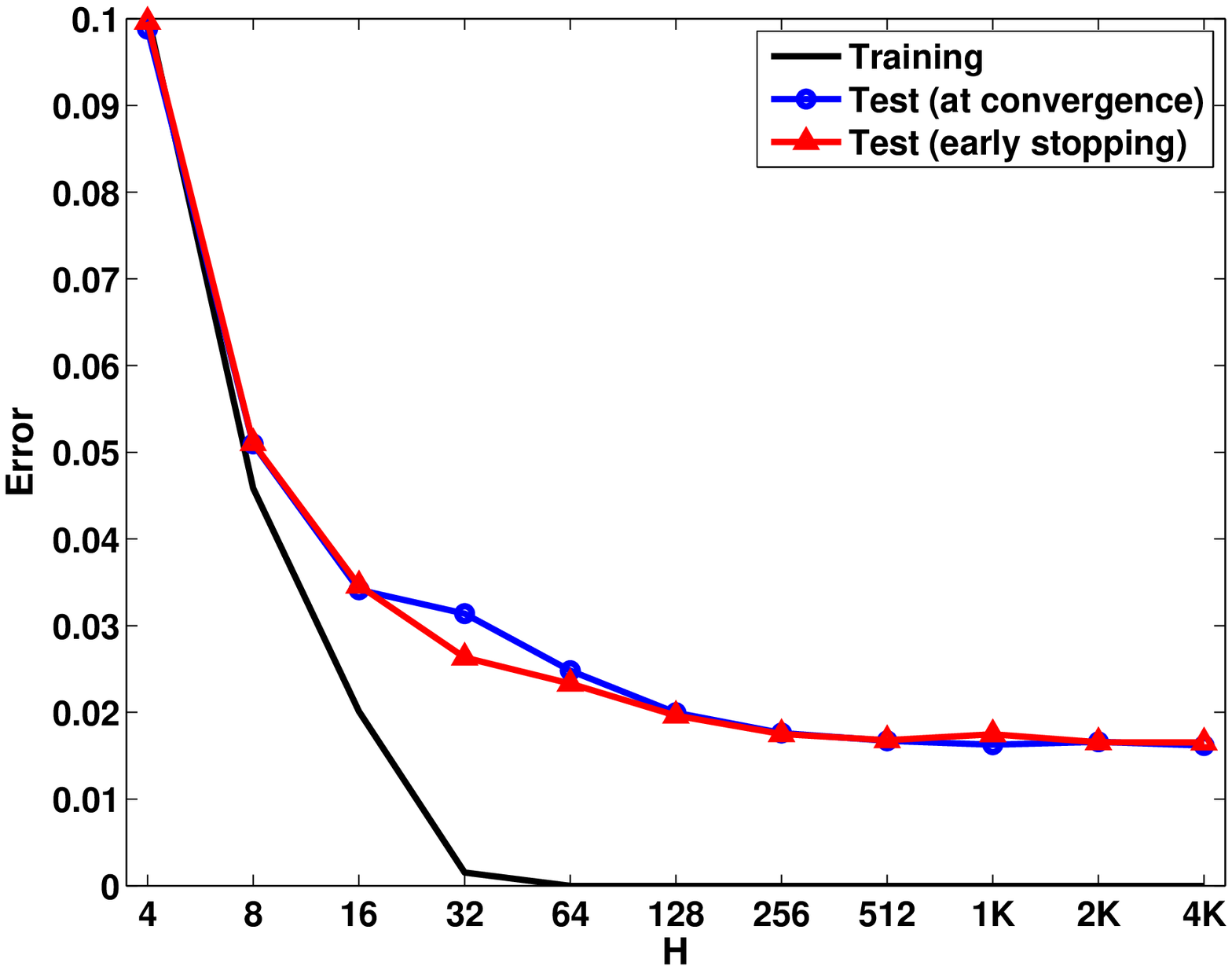}
\hspace{0.1in}
\setlength{\epsfxsize}{0.45\textwidth}
\epsfbox{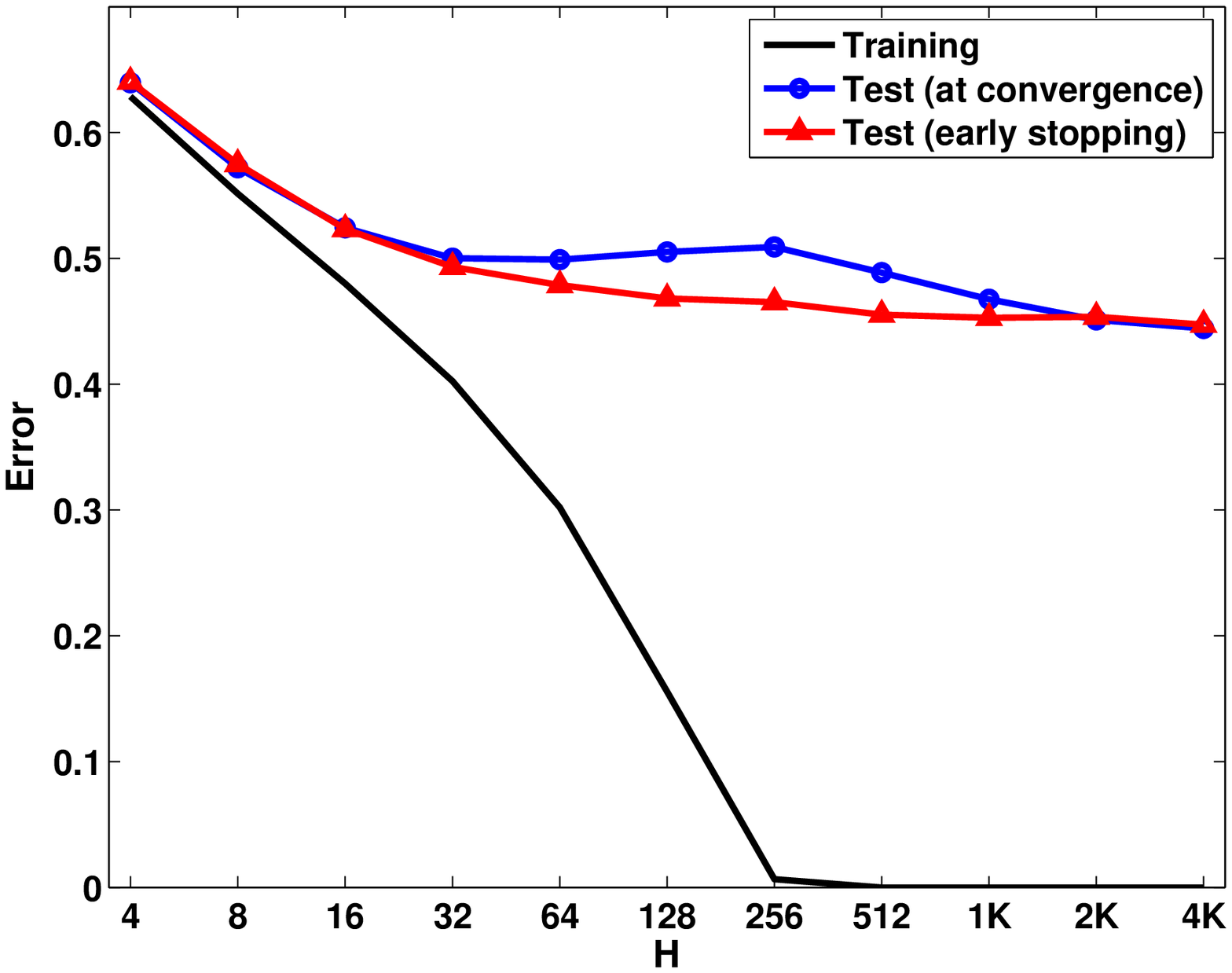}
}
\begin{picture}(0,0)(0,0)
\put(85, 155){MNIST}
\put(280, 155){CIFAR-10}
\end{picture}
\vspace{0.1in}
\caption{\small The training error and the test error based on different stopping
  criteria when 2-layer NNs with different number of hidden units
  are trained on MNIST and CIFAR-10. Images in both datasets are downsampled
  to 100 pixels. The size of the training set is 50000 for MNIST and 40000 for CIFAR-10.
  The early stopping is based on the error on a validation set
  (separate from the training and test sets) of size 10000. The training was
  done using stochastic gradient descent with momentum and mini-batches
  of size 100. The network was initialized with weights generated randomly from
  the Gaussian distribution. The initial step size and momentum were set to 0.1 and 0.5
  respectively. After each epoch, we used the update rule $\mu^{(t+1)}=0.99\mu^{(t)}$
  for the step size and $m^{(t+1)}=\min\{0.9,m^{(t)}+0.02\}$ for the momentum. 
  \label{fig:intro}
}
\end{figure*}

We also further tested this phenomena under some artificial
mutilations to the data set.  First, we wanted to artificially ensure
that the approximation error was indeed zero and does not decrease as
we add more units.  To this end, we first trained a network with a
small number $H_0$ of hidden units ($H_0=4$ on MNIST and $H_0=16$ on
CIFAR) on the entire dataset (train+test+validation).  This network
did have some disagreements with the correct labels, but we then
switched all labels to agree with the network creating a ``censored''
data set.  We can think of this censored data as representing an
artificial source distribution which can be exactly captured by a
network with $H_0$ hidden units.  That is, the approximation error is zero
for networks with at least $H_0$ hidden units, and so does not
decrease further. Still, as can be seen in the middle row of
Figure~\ref{fig:netsize}, the test error continues decreasing even
after reaching zero training error.

Next, we tried to force overfitting by adding random label noise to
the data.  We wanted to see whether now the network will use its
higher capacity to try to fit the noise, thus hurting generalization.
However, as can be seen in the bottom row of Figure~\ref{fig:netsize},
even with five percent random labels, there is no significant
overfitting and test error continues decreasing as network size
increases past the size required for achieving zero training error.

\begin{figure*}[t!]
\hbox{ \centering
\setlength{\epsfxsize}{0.455\textwidth}
\epsfbox{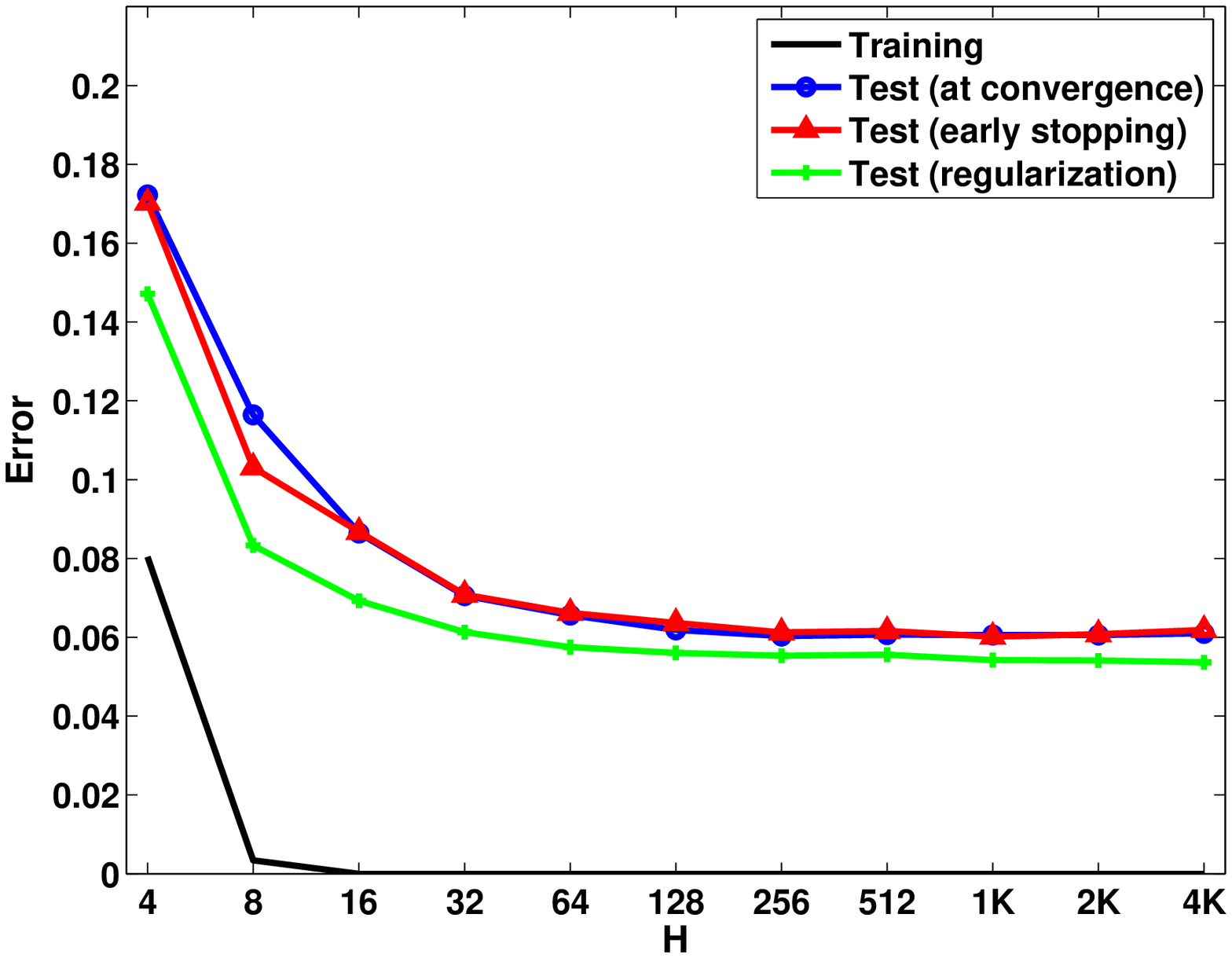}
\hspace{0.1in}
\setlength{\epsfxsize}{0.45\textwidth}
\epsfbox{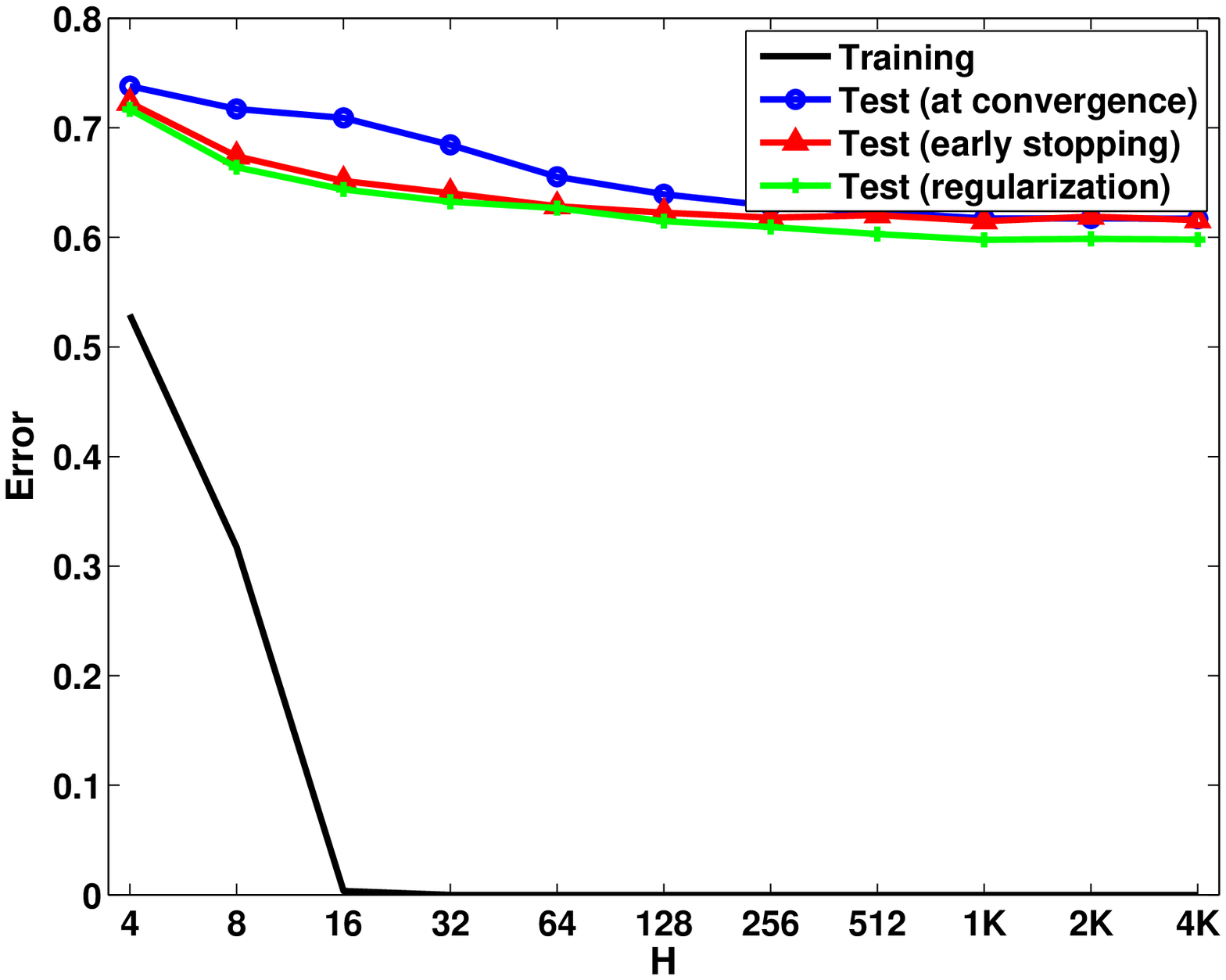}
}

\hbox{ \centering
\setlength{\epsfxsize}{0.455\textwidth}
\epsfbox{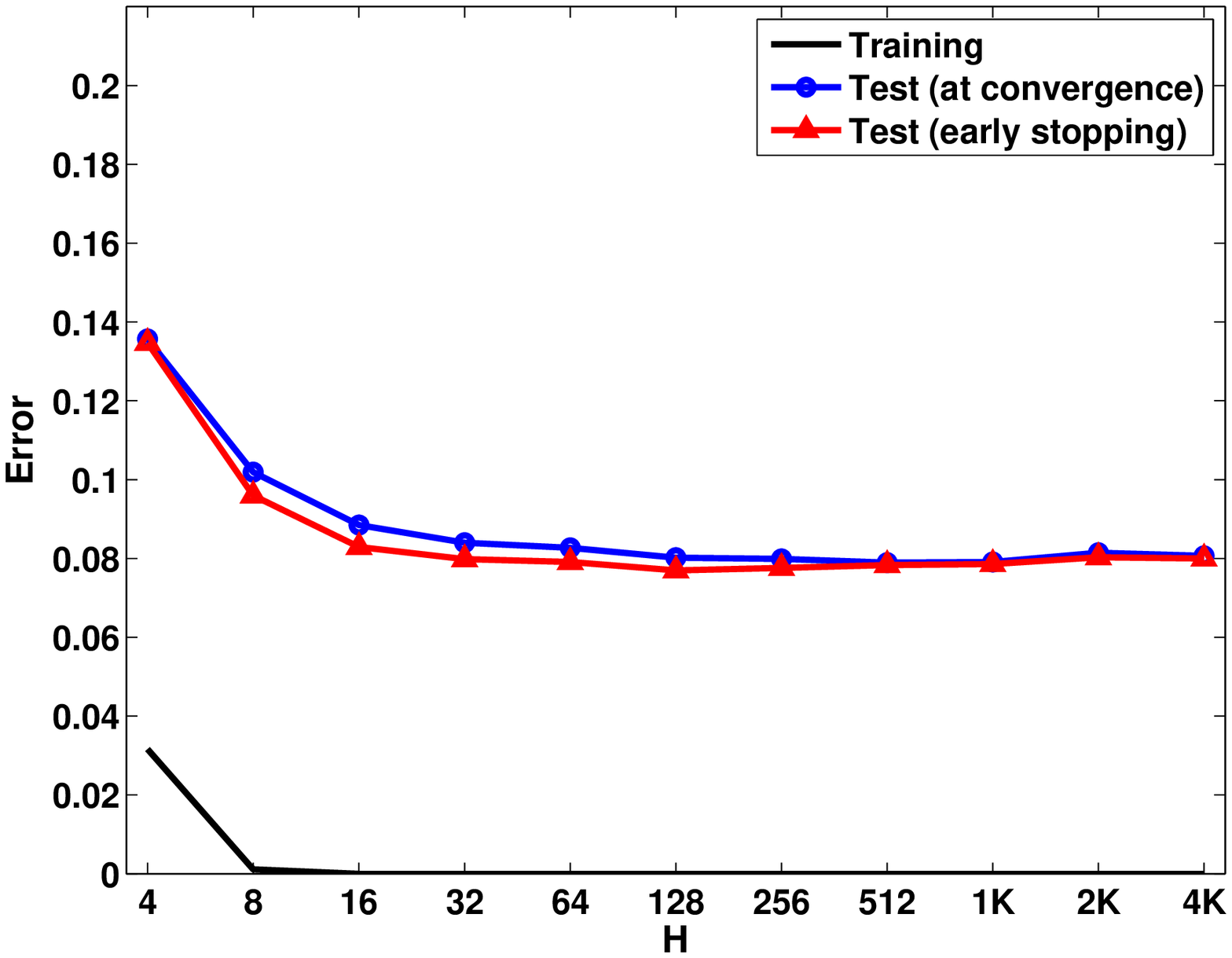}
\hspace{0.1in}
\setlength{\epsfxsize}{0.45\textwidth}
\epsfbox{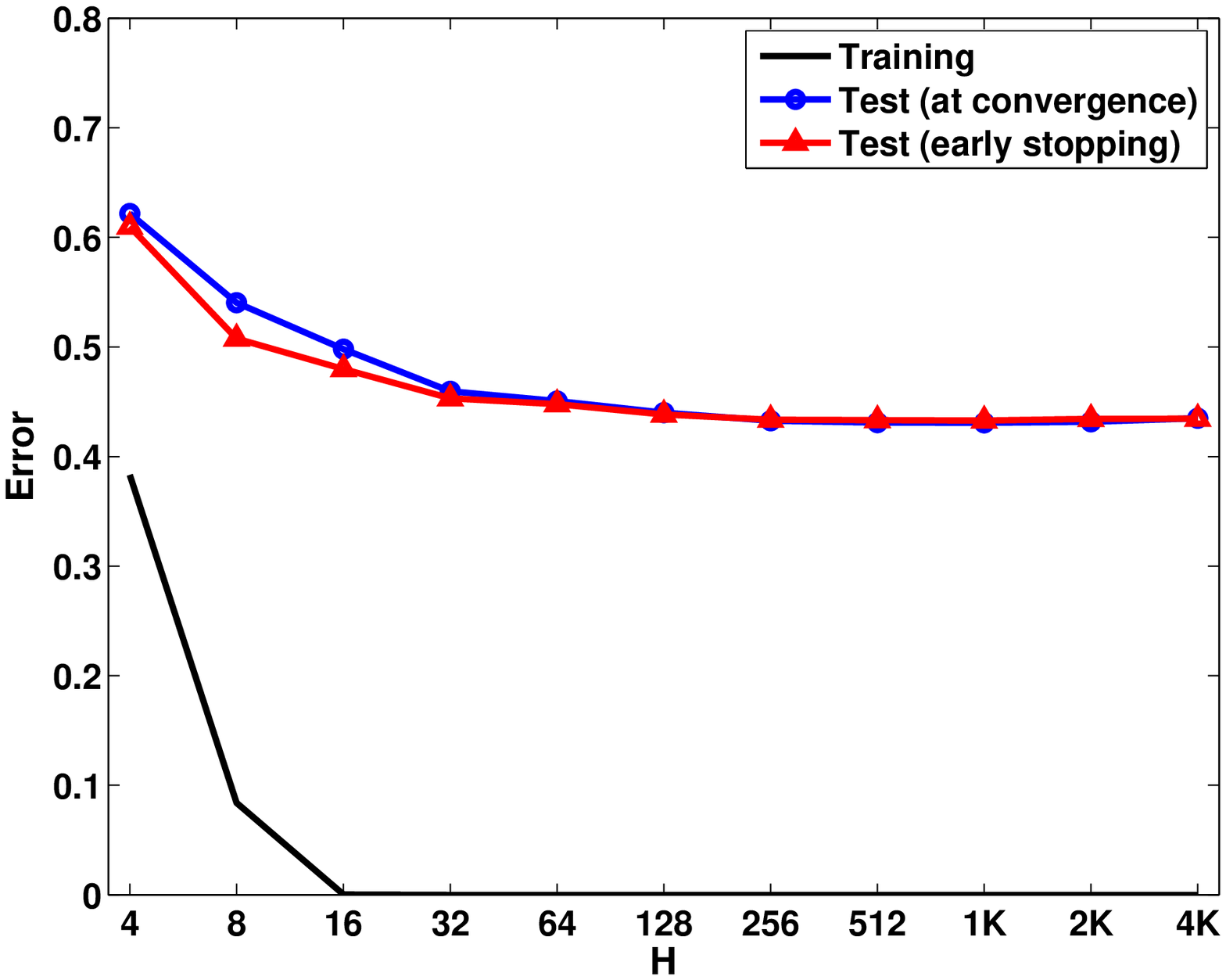}
}

\hbox{ \centering
\setlength{\epsfxsize}{0.455\textwidth}
\epsfbox{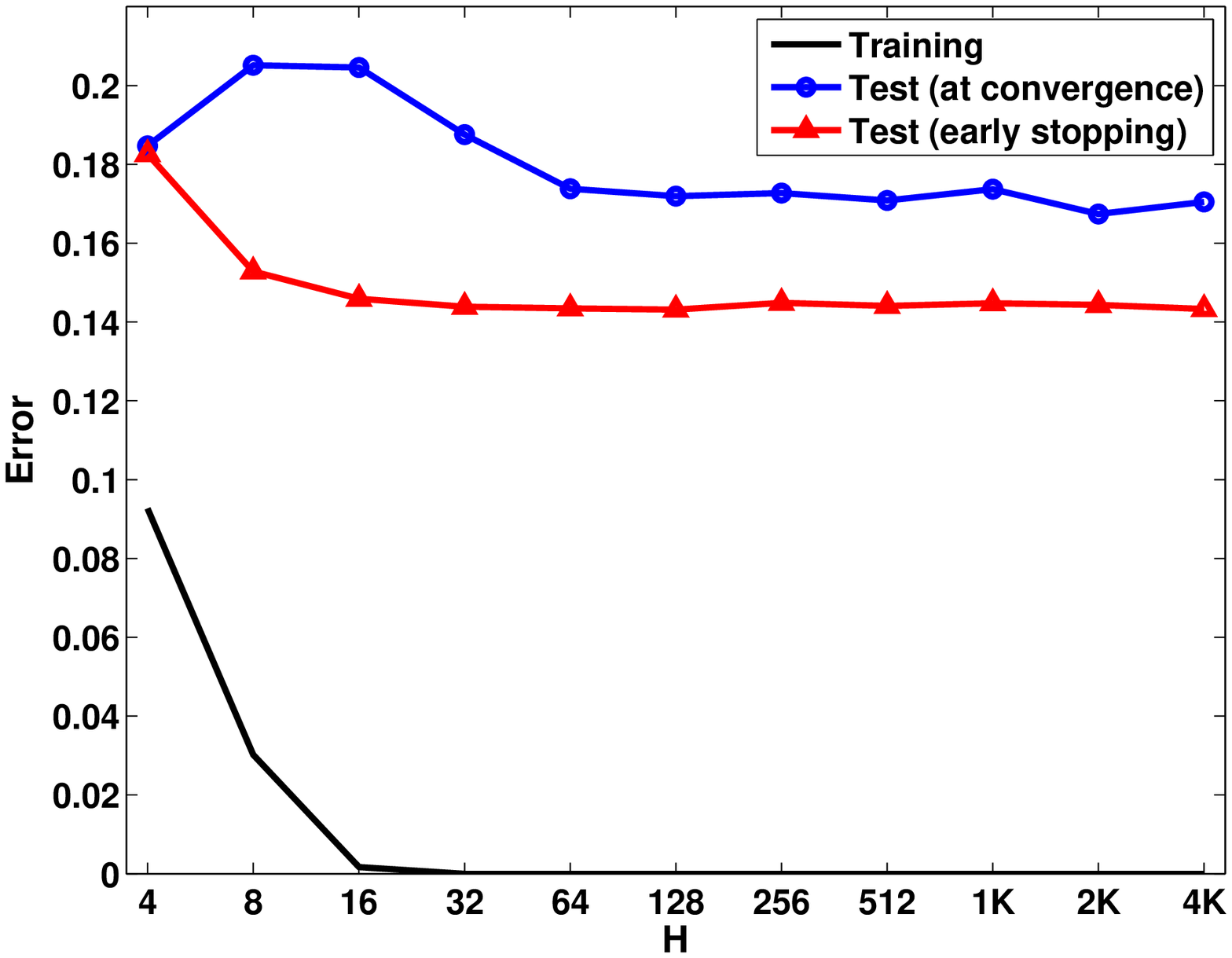}
\hspace{0.1in}
\setlength{\epsfxsize}{0.45\textwidth}
\epsfbox{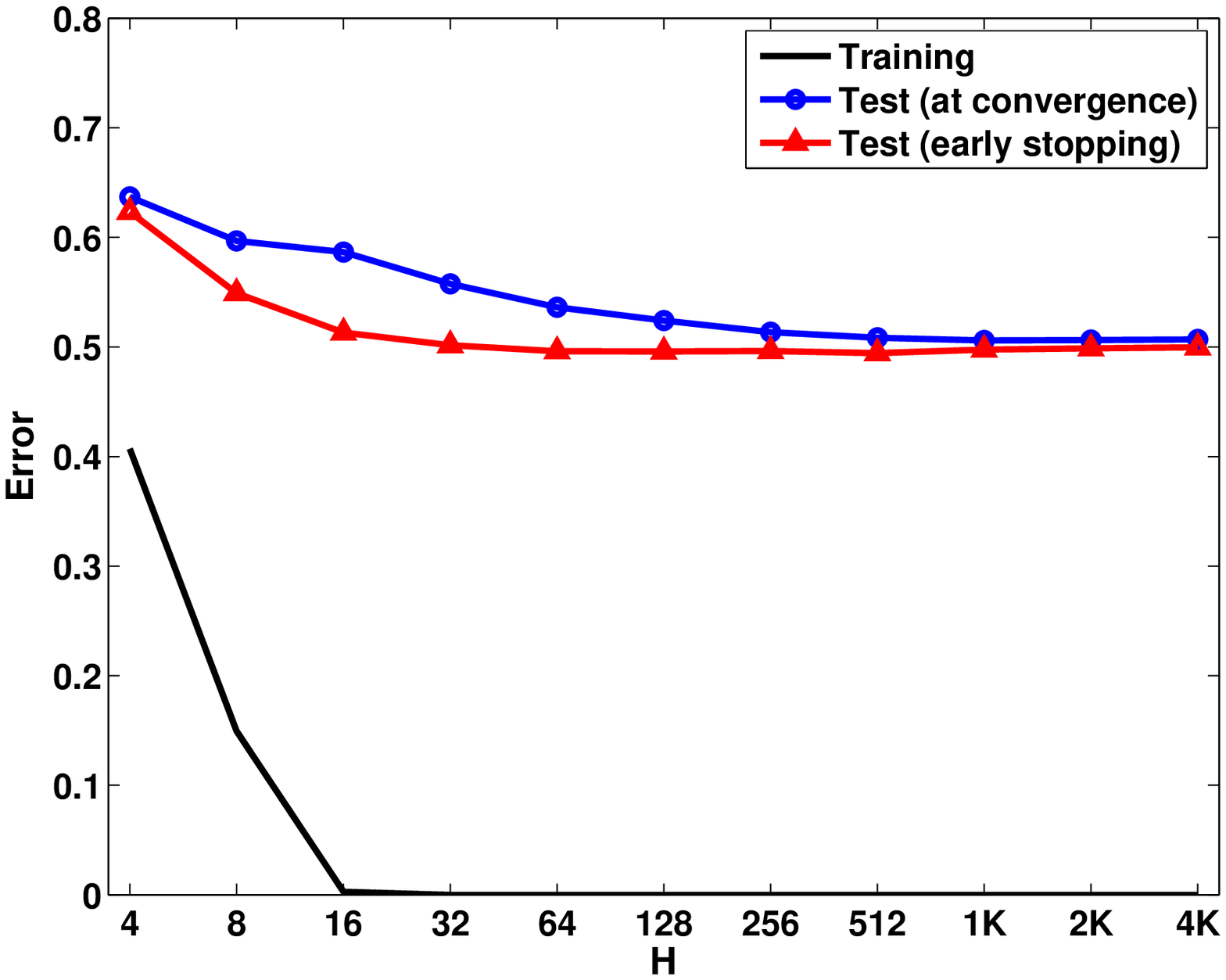}
}
\begin{picture}(0,0)(0,0)
\put(70, 445){Small MNIST}
\put(260, 445){Small CIFAR-10}
\end{picture}
\caption{\small The training error and the test error based on different stopping
  criteria when 2-layer NNs with different number of hidden units
  are trained on small subsets of MNIST and CIFAR-10. Images in both
  datasets are downsampled to 100 pixels. The sizes of the training
  and validation sets are 2000 for both MNIST and CIFAR-10 and the early stopping
  is based on the error on the validation set. The top plots are
  the errors for the original datasets with and without explicit regularization.The best
  weight decay parameter is chosen based on the validation error.
  The middle plots are on the censored data set that is constructed
  by switching all the labels to agree with the predictions of a trained network
  with a small number $H_0$ of hidden units  $H_0=4$ on MNIST and $H_0=16$
  on CIFAR-10) on the  entire dataset (train+test+validation).
  The plots on the bottom are also for the censored data except we also add 5 percent
  noise to the labels by randomly changing 5 percent of the labels. 
  The optimization method is the same as the in Figure 1.  The results in this figure are the average error over 5 random repetitions.
\label{fig:netsize}
}
\end{figure*}

What is happening here?  A possible explanation is that the
optimization is introducing some implicit regularization.  That is, we are
implicitly trying to find a solution with small ``complexity'', for
some notion of complexity, perhaps norm.  This can explain why we do
not overfit even when the number of parameters is huge.  Furthermore,
increasing the number of units might allow for solutions that actually
have lower ``complexity'', and thus generalization better.  Perhaps an
ideal then would be an infinite network controlled only through this
hidden complexity.

We want to emphasize that we are not including any explicit
regularization, neither as an explicit penalty term nor by modifying
optimization through, e.g., drop-outs, weight decay, or with
one-pass stochastic methods.  We are using a stochastic method, but we
are running it to convergence---we achieve zero surrogate loss and zero training error. In fact, we also tried training using batch conjugate
gradient descent and observed almost identical behavior.  But it seems
that even still, we are not getting to some random global
minimum---indeed for large networks the vast majority of the many
global minima of the training error would horribly overfit.  Instead,
the optimization is directing us toward a ``low complexity'' global
minimum.

Although we do not know what this hidden notion of complexity is, as a
final experiment we tried to see the effect of adding explicit
regularization in the form of weight decay.  The results are shown in the top
row of figure~\ref{fig:netsize}. There is a slight improvement in generalization
but we still see that increasing the network size helps generalization.

\section{A Matrix Factorization Analogy}

To gain some understanding at what might be going on, let us consider
a slightly simpler model which we do understand much better.  Instead of
rectified linear activations, consider a feed-forward network with a
single hidden layer, and {\em linear} activations, i.e.:
\begin{equation}
  \label{eq:ykl}
  y[j] = \sum_{h=1}^H v_{hj} \inner{\vu_h}{\vx}.
\end{equation}
This is of course simply a matrix-factorization model, where $\vy=\mW\vx$
and $\mW=\mV\mU\T$.  Controlling capacity by limiting the number of hidden
units exactly corresponds to constraining the rank of $\mW$,
i.e.~biasing toward low dimensional factorizations.  Such a low-rank
inductive bias is indeed sensible, though computationally intractable
to handle with most loss functions.

However, in the last decade we have seen much success for learning
with low {\em norm} factorizations.  In such models, we do not
constrain the inner dimensionality $H$ of $\mU,\mV$, and instead only
constrain, or regularize, their norm.  For example, constraining the
Frobenius norm of $\mU$ and $\mV$ corresponds to using the {\em
  trace-norm} as an inductive bias \citep{Srebro04}:
\begin{equation}\label{eq:tracenorm}
  \norm{\mW}_{tr} = \min_{\mW=\mV\mU\T} \frac{1}{2}(\norm{\mU}_F^2+\norm{\mV}_F^2).
\end{equation}
Other norms of the factorization lead to different regularizers.  

Unlike the rank, the trace-norm (as well as other factorization norms)
is convex, and leads to tractable learning problems
\citep{Fazel01,Srebro04}.  In fact, even if learning is done by a
local search over the factor matrices $\mU$ and $\mV$ (i.e.~by a local
search over the weights of the network), if the dimensionality is high
enough and the norm is regularized, we can ensure convergence to a
global minima \citep{Burer06}. This is in stark contrast to the
dimensionality-constrained low-rank situation, where the limiting
factor is the number of hidden units, and local minima are abundant
\citep{Srebro03}.

Furthermore, the trace-norm and other factorization norms are
well-justified as sensible inductive biases.  We can ensure
generalization based on having low trace-norm, and a low-trace norm
model corresponds to a realistic factor model with many factors of
limited overall influence.  In fact, empirical evidence suggests that
in many cases low-norm factorization are a more appropriate inductive
bias compared to low-rank models.

We see, then, that in the case of linear activations (i.e.~matrix
factorization), the norm of the factorization is in a sense a better
inductive bias than the number of weights: it ensures generalization,
it is grounded in reality, and it explain why the models can be learned
tractably.  

Let us interpret the experimental results of Section \ref{sec:netsize} in this
light.  Perhaps learning is succeeding not because there is a good
representation of the targets with a small number of units, but rather
because there is a good representation with small overall norm, and
the optimization is implicitly biasing us toward low-norm models.
Such an inductive bias might potentially explain both the
generalization ability {\em and} the computational tractability of
learning, even using local search.

Under this interpretation, we really should be using infinite-sized
networks, with an infinite number of hidden units.  Fitting a finite
network (with implicit regularization) can be viewed as an
approximation to fitting the ``true'' infinite network.  This
situation is also common in matrix factorization: e.g., a very
successful approach for training low trace-norm models, and other
infinite-dimensional bounded-norm factorization models, is to
approximate them using a finite dimensional representation
\cite{rennie2005fast,srebro2010collaborative}.  The finite
dimensionality is then not used at all for capacity (statistical
complexity) control, but purely for computational reasons.  Indeed, increasing
the allowed dimensionality generally improves generalization performance, as it
allows us to better approximate the true infinite model.

\section{Infinite Size, Bounded Norm Networks}

In this final section, we consider a possible model for infinite sized
norm-regularized networks.  Our starting point is that of global weight decay,
i.e.~adding a regularization term that penalizes the sum of squares of
all weights in the network, as might be approximately introduced by
some implicit regularization.  Our result in this Section is that this
global $\ell_2$ regularization is equivalent to a  {\em Convex Neural
Network} (Convex NN; \cite{Bengio05})---an infinite network with $\ell_1$
regularization on the top layer.  Note that such models are rather
different from infinite networks with $\ell_2$ regularization on the
top layer, which reduce to linear methods with some specific
kernel~\citep{Cho09,Bach14}. Note also that our aim here is to explain what
neural networks are doing instead of trying to match the performance 
of deep models with a known shallow model as done by, e.g., \citet{Lu14}.

For simplicity, we will focus on single output networks ($k=1$), i.e.~networks
which compute a function $f:\R^d\rightarrow\R$.  We first consider
finite two-layer networks (with a single hidden layer) and show that
$\ell_2$ regularization on both layers is equivalent to and $\ell_2$
constraint on each unit in the hidden layer, and $\ell_1$
regularization on the top unit:
\begin{theorem}\label{thm:L2toL1}
Let $L$ be a loss function and $D=(\vx_t,y_t)_{t=1}^{n}$ be training examples.
\begin{align}\label{eq:weight-decay-net}
 \argmin{\vv\in\RR^{H},(\vu_h)_{h=1}^{H}}\left(
\sum_{t=1}^{n}L\left(y_t, \sum\nolimits_{h=1}^{H}v_h[\inner{\vu_h}{\vx_t}]_+\right)+\frac{\lambda}{2}\sum_{h=1}^{H}\left(\norm{\vu_h}^2+|v_h|^2\right)
\right),
\end{align}
is the same as
\begin{align}
\label{eq:learned-L1}
 \argmin{\vv\in\RR^{H},(\vu_h)_{h=1}^{H}}\Biggl(
&\sum_{t=1}^{n}L\left(y_t,
 \sum\nolimits_{h=1}^{H}v_h[\inner{\vu_h}{\vx_t}]_+\right)+\lambda\sum_{h=1}^{H}|v_h|\Biggr),\\
& \text{subject to}\quad \|\vu_h\|\leq 1\quad (h=1,\ldots,H).\notag
\end{align}
\end{theorem}
\begin{proof}
By the inequality between the arithmetic and
geometric means, we have
\begin{align*}
\frac{1}{2}\sum_{h=1}^{H}\left(\|\vu_h\|^2+|v_h|^2\right)\geq \sum_{h=1}^{H}\|\vu_h\|\cdot|v_h|.
\end{align*}
The right-hand side can always be attained without changing the
 input-output mapping by the rescaling
$\tilde{\vu}_h=\sqrt{|v_h|/\|\vu_h\|}\cdot\vu_h$
 and $\tilde{v}_h=\sqrt{\|\vu_h\|/|v_h|}\cdot v_h$.
The reason we can rescale the weights without changing 
the input-output mapping is that the rectified linear unit is
 piece-wise linear and the piece a hidden unit is on is invariant to 
rescaling of the weights. 
Finally, since the right-hand side of the
 above inequality is invariant to rescaling, we can always choose
the norm $\|\vu_h\|$ to be bounded by one.
\end{proof}

Now we establish a connection between the $\ell_1$-regularized network
\eqref{eq:learned-L1}, in which we learn the input to hidden
connections, to convex NN \citep{Bengio05}.

First we recall the definition of a convex NN. Let $\calU$ be a fixed
``library'' of possible weight vectors $\{\vu\in\RR^{d}: \|\vu\|\leq
1\}$ and $\mu_+$ and $\mu_-$ be positive (unnormalized) measures over
$\calU$ representing the positive and negative part of the weights of
each unit. Then a ``convex neural net'' is given by predictions of the form
\begin{equation}
  \label{eq:predconvNN}
  y=\int_{\calU}
\left({\rm d}\mu_{+}(\vu)-
{\rm d}\mu_{-}(\vu)\right)[\inner{\vu}{\vx}]_+
\end{equation}
with regularizer (i.e.~complexity measure) $\mu_{+}(\calU)+\mu_{-}(\calU)$.  This is simply an infinite generalization of network with hidden units $\calU$ and $\ell_1$ regularization on the second layer: if $\calU$ is finite, \eqref{eq:predconvNN} is equivalent to
\begin{equation}
  \label{eq:predconvNNfinite}
  y = \sum_{\vu\in\calU} v(\vu)[\inner{\vu}{\vx}]_+
\end{equation}
with $v(\vu):=\mu_{+}(\vu)-\mu_{-}(\vu)$ and regularizer $\norm{v}_1=\sum_{\vu\in\calU}|v(\vu)|$.  Training a convex NN is then given by:
\begin{align}
\label{eq:fixed-L1}
 \argmin{\mu_+,\mu_-}\left(
\sum_{t=1}^{n}L\Bigl(y_t,\int_{\calU}
\left({\rm d}\mu_{+}(\vu)-
{\rm d}\mu_{-}(\vu)\right)[\inner{\vu}{\vx_t}]_+\Bigr)+\lambda
\left(\mu_{+}(\calU)+\mu_{-}(\calU)\right)
\right).
\end{align}
Moreover, even if $\calU$ is infinite and even continuous, there will always be an optimum of \eqref{eq:fixed-L1} which is a discrete measure with support at most $n+1$ \citep{Rosset07}.  That is, \eqref{eq:fixed-L1} can be equivalently written as:
\begin{align}\label{eq:fixed-L1-finite}
\argmin{\vu_1,\ldots,\vu_{n+1}\in\calU,\vv\in\R^{n+1}}\left(
\sum_{t=1}^{n}L\Bigl(y_t,\sum_{h=1}^{n+1}
v_h[\inner{\vu_h}{\vx_t}]_+\Bigr)+\lambda \norm{\vv}_1
\right),
\end{align}
which is the same as \eqref{eq:learned-L1}, with $\calU=\left\{ \vu\in\R^d \,\middle|\, \norm{\vu}\leq 1 \right\}$.

The difference between the network \eqref{eq:learned-L1}, and the
infinite network \eqref{eq:fixed-L1} is in {\em learning} versus {\em
  selecting} the hidden units, and in that in \eqref{eq:fixed-L1} we
have no limit on the number of units used.  That is, in
\eqref{eq:fixed-L1} we have all possible units in $\calU$ available to
us, and we merely need to select which we want to use, without any
constraint on the {\em number} of units used, only the over $\ell_1$
norm.  But the equivalence of \eqref{eq:fixed-L1} and
\eqref{eq:fixed-L1-finite} establishes that as long as the number of
allowed units $H$ is large enough, the two are equivalent:

\begin{corollary}
  As long as $H>n$, the weight-decay regularized network
  \eqref{eq:weight-decay-net} is equivalent to the convex neural net \eqref{eq:fixed-L1} with $\calU=\left\{ \vu\in\R^d \,\middle|\, \norm{\vu}\leq 1 \right\}$.
\end{corollary}

In summary, learning and selecting is equivalent if we have sufficiently
many hidden units and Theorem \ref{thm:L2toL1} gives an alternative
justification for employing $\ell_1$ regularization when the input to
hidden weights are fixed and normalized to have unit norm, namely, it is
equivalent to $\ell_2$ regularization, which can be achieved by weight
decay or implicit regularization via stochastic gradient descent.

The above equivalence holds also for networks with multiple output
units, i.e. $k>1$, where the $\ell_1$ regularization on $v$ is
replaced with the group lasso regularizer
$\sum_{h=1}^{H}\norm{\vv_h}$. Indeed, for matrix factorizations
(i.e.~with linear activations), such a group-lasso regularized
formulation is known to be equivalent to the trace norm
\eqref{eq:tracenorm} \citep[see][]{ArgEvgPon07}.

\bibliography{iclr2015}
\bibliographystyle{iclr2015}
\section*{Appendix}
For the convenience of the reader, we formalize here the hardness of
learning feed-forward neural network mentioned in the Introduction.
The results are presented in a way that is appropriate for
feed-forward networks with RELU activations, but they are really a
direct implication of recent results about learning intersections of
halfspaces.  For historical completeness we note that hardness of
learning logarithmic depth networks was already established by
\citet{kearns1994}, and that the more recent results we
discuss here \citep{klivans2006cryptographic,Daniely14} establish also hardness of
learning depth two networks, subject to perhaps simpler cryptographic
assumptions.  The presentation and construction here is similar to
that of \citet{livni14}.

\paragraph{Question}
Is there a sample complexity function $m(D,H)$ and an algorithm
$\calA$ that takes as input
$\left\{(x_i,y_i)\right\}_{i=1,\dots,M},x_i\in\{\pm 1\}^D,y\in\pm 1$
and returns a description of a function $f:\{\pm 1\}^D \rightarrow \pm
1$ such that the following is true:

For any $D$, any $H$ and any distribution $\calD(x,y)$ over $x\in\{\pm
1\}^D,y\in \pm 1$, if:
\begin{enumerate}
\item [(1)] There exists a feed-forward neural network with RELU
  actications with $D$ inputs and $H$ hidden units implementing a
  function $g:\R^D\rightarrow\R$ such that $\Pr_{(x,y)\sim\calD}[
  yg(x)>0 ] =1$ (i.e. for $(x,y)\sim\calD$, the label can be perfectly
  predicted by a network of size $H$).
\item [(2)] The input to $\calA$ is drawn i.i.d.~from $\calD$.
\item [(3)] $M\geq m(D,H)$ (i.e.~$\calA$ is provided with enough
  training data).
\end{enumerate}
then
\begin{enumerate}
\item[(1)] With probability at least $1/2$, algorithm $\calA$ returns
  a function $f$ such that $\mathbb{P}_{x,y\sim
    \calD}\left[yf(x)<0\right] < 1/4$ (that is, at least half the time
  the algorithm succeeds in learning a function with non-trivial error).
\item [(3)] $m(D,H)\leq \text{poly}(D,H)$ for some $\text{poly}(D,H)$
  (i.e.~the sample complexity required by the algorithm is polynomial
  in the network size---if we needed a super-polynomial number of
  samples, we would have no hope of learning in polynomial time).
\item [(4)] The function $f$ that corresponds to the description
  returned by $\calA$ can be computed in time $\text{poly}(D,H)$ from
  its description (i.e.~the representation used by the learner can be
  a feed-forward network of any size polynomial in $H$ and $D$,
  or of any other representation that can be efficiently computed).
\item [(5)] $\calA$ runs in time $\text{poly}(D,H,M)$
\end{enumerate}
\begin{theorem*}
Subject for the cryptographic assumptions in \cite{Daniely14}, there is no algorithm that satisfies the conditions in the above question.
\end{theorem*}
In fact, there is no algorithm satisfying the conditions even if we
require that the labels can be perfectly predicted by a network with a
single hidden layer with any super-constant, e.g.~$\log(D)$, number of
hidden units.
\begin{proof}
We show that every intersection of $k=\omega(1)$ homogeneous halfspaces
over $\{\pm 1\}^n$ with normals in $\{\pm1\}$
can be realized with unit margin by a feed-fowrad neural networks with
$H=2k$ hidden units in a single hidden layer.
For each hyperplane $\inner{w_i}{x}>0$, where
$w_i\in \{\pm 1\}^D$, we include two units in the hidden layer:
$g^+_{i}(x) = [\inner{w_i}{x}]_+$ and $g^-_i(x) = [\inner{w_i}{x}-1]_+$.
We set all incoming weights of the output node to be $1$.
Therefore, this network is realizing the following function:
$$
f(x) = \sum_{i=1}^k \left([\inner{w_i}{x}]_+ - [\inner{w_i}{x}-1]_+\right)
$$
Since all inputs and all weights are integer, the outputs of the first
layer will be integer, $\left([\inner{w_i}{x}]_+ -
  [\inner{w_i}{x}-1]_+\right)$ will be zero or one, and $f$ realizes
the intersection of the $k$ halfspaces with unit margin. 
Hence, the hypothesis class of neural intersection of 
$k/2$ halfspaces is a subset of hypothesis class of feed-forward neural networks with $k$
hidden units in a single hidden layer. We complete the proof by applying Theorem 5.4 in \cite{Daniely14} which states that for any $k=\omega(1)$, subject for the cryptographic assumptions in \cite{Daniely14}, the hypothesis class of intersection of
  homogeneous halfspaces over $\{\pm 1\}^n$ with normals in $\{\pm
  1\}$ is not efficiently PAC learnable (even
  improperly)\footnote{Their Theorem 5.4 talks about unrestricted
    halfspaces, but the construction in Section 7.2 uses only data in
    $\{ \pm 1 \}^D$ and halfspaces specified by $\langle w,x\rangle
    >0$ with $w\in\{\pm 1\}^D$}.
\end{proof}
We proved here that even for $H=\omega(1)$ no algorithm can satisfy
the condition in the question. A similar result can be shown for
$H=\text{poly}( D)$ subject to weaker cryptographic assumptions in
\cite{klivans2006cryptographic}.

The Theorem tells us not only that we cannot expect to fit a small
network to data even if the data is generated by the network (since
doing so would give us an efficient learning algorithm, which
contradicts the Theorem), but that we also can't expect to learn by
using a much larger network.  That is, even if we know that labels can
be perfectly predicted by a small network, we cannot expect to have a
learning algorithm that learns a much larger (but poly sized) network
that will have non-trivial error.  In fact, being representable by a
small network is not enough to ensure tractable learning no matter
what representation the learning algorithm uses (e.g.~a much larger
network, a mixture of networks, a tree over networks, etc).  This is a
much stronger statement than just saying that fitting a network to
data is $NP$-hard.  Also, precluding the possibility of tractable
learning if the labels are exactly explained by some small unknown
network of course also precludes the possibility of achieving low
error when the labels are only approximately explained by some small
unknown network (i.e.~of noisy or ``agnostic'' learning).

\end{document}